\documentclass[compsoc,10pt,journal]{IEEEtran}
\usepackage{mathtools}
\usepackage{amsmath,amsfonts, bm, amssymb}
\usepackage{algorithmic}
\usepackage{algorithm}
\usepackage{array}
\usepackage[caption=false,font=normalsize,labelfont=sf,textfont=sf]{subfig}
\usepackage{textcomp}
\usepackage{stfloats}
\usepackage{url}
\usepackage{verbatim}
\usepackage{graphicx}
\usepackage{cite}
\usepackage{pifont}
\usepackage{mathrsfs}
\usepackage{amsthm}
\usepackage{adjustbox}
\usepackage{tabularray}
\usepackage{booktabs}
\usepackage[rightcaption]{sidecap}
\usepackage{soul}
\usepackage{xcolor}
\sethlcolor{orange}

\newtheorem{theorem}{Theorem}

\newcommand{\tick}{\ding{51}} 
\newcommand{\cross}{\ding{55}} 

\begin{document}

\author{Shahin Hakemi, Naveed Akhtar,~\IEEEmembership{Member,~IEEE}, Ghulam Mubashar Hassan,~\IEEEmembership{Senior Member,~IEEE}, Ajmal~Mian,~\IEEEmembership{Senior Member,~IEEE}
\thanks{Shahin Hakemi, Ghulam Mubashar Hassan, and Ajmal Mian are with the Department of Computer Science and Software Engineering, University of Western Australia, Perth, WA 6009, Australia (e-mail:
shahin.hakemi@research.uwa.edu.au; ghulam.hassan@uwa.edu.au; ajmal.mian@uwa.edu.au).}
\thanks{Naveed Akhtar is with the School of Computing and Information Systems, University of Melbourne, Melbourne, VIC 3010, Australia (e-mail: naveed.akhtar1@unimelb.edu.au).}
\thanks{Manuscript received April 19, 2021; revised August 16, 2021.}}

\title{Single-weight Model Editing for Post-hoc Spurious Correlation Neutralization}
\IEEEtitleabstractindextext{
\begin{abstract}
\justifying
Neural network training tends to exploit the simplest features as shortcuts to greedily minimize training loss. However, some of   these features might be spuriously correlated with the target labels, leading to incorrect predictions by the model. Several methods have been proposed to address this  issue. Focusing on suppressing the spurious correlations with model training, they  not only incur additional training cost, but also have  limited practical utility as the model misbehavior due to spurious relations is usually discovered after its deployment. It is also often overlooked that spuriousness is a subjective notion. Hence, the precise questions that must be investigated are; to what degree a feature is spurious, and how we can {\em proportionally} distract the model's attention from it for reliable prediction. To this end, we propose a method that enables post-hoc neutralization of spurious feature impact, controllable to an  arbitrary degree. We conceptualize spurious features  as fictitious sub-classes  within the  original classes, which can be eliminated by a class removal scheme. We then propose a  unique precise class removal technique that makes a single-weight modification, which entails negligible performance compromise for the remaining classes. We perform extensive experiments, demonstrating that by editing just a single weight in a post-hoc manner, our method achieves highly competitive, or better performance against the state-of-the-art methods. \textit{Our implementation will be made public after acceptance.}
\end{abstract}

\begin{IEEEkeywords}
model editing, machine unlearning, spurious correlation mitigation, post-hoc algorithms
\end{IEEEkeywords}}

\maketitle
\IEEEdisplaynontitleabstractindextext
\section{Introduction}
\label{sec:intro}
\IEEEPARstart{A}{rtificial} 
Neural Networks (ANNs) that employ Empirical Risk Minimization (ERM) \cite{vapnik1999overview}  are prone to  correlate {spurious} features to target labels \cite{geirhos2020shortcut, liu2022avoiding, yang2023mitigating, he2025eva}. Residing in training data, such features often provide shortcuts to minimize loss, causing over-reliance of the model on them for inference~\cite{zhang2022correct, ye2024spurious, han2024improving, zheng2025neurontune}. This leads to poor model generalization. Currently, the prevalent paradigm of suppressing  learning of \textit{spurious feature-target label correlation} - aka \textit{spurious correlation} - is robust model learning \cite{ludan2023explanation, nam2022spread, pagliardini2022agree, labonte2024towards, bombari2025spurious}, which either requires a subsequent model retraining~\cite{liu2021just, asgari2022masktune}, or training the model robustly right from scratch~\cite{deng2024robust, wang2024learning, akhtar2024roboss, kim2024improving, yang2025regulating}.
\begin{figure}[t]
    \centering
    \includegraphics[width=0.4\textwidth]{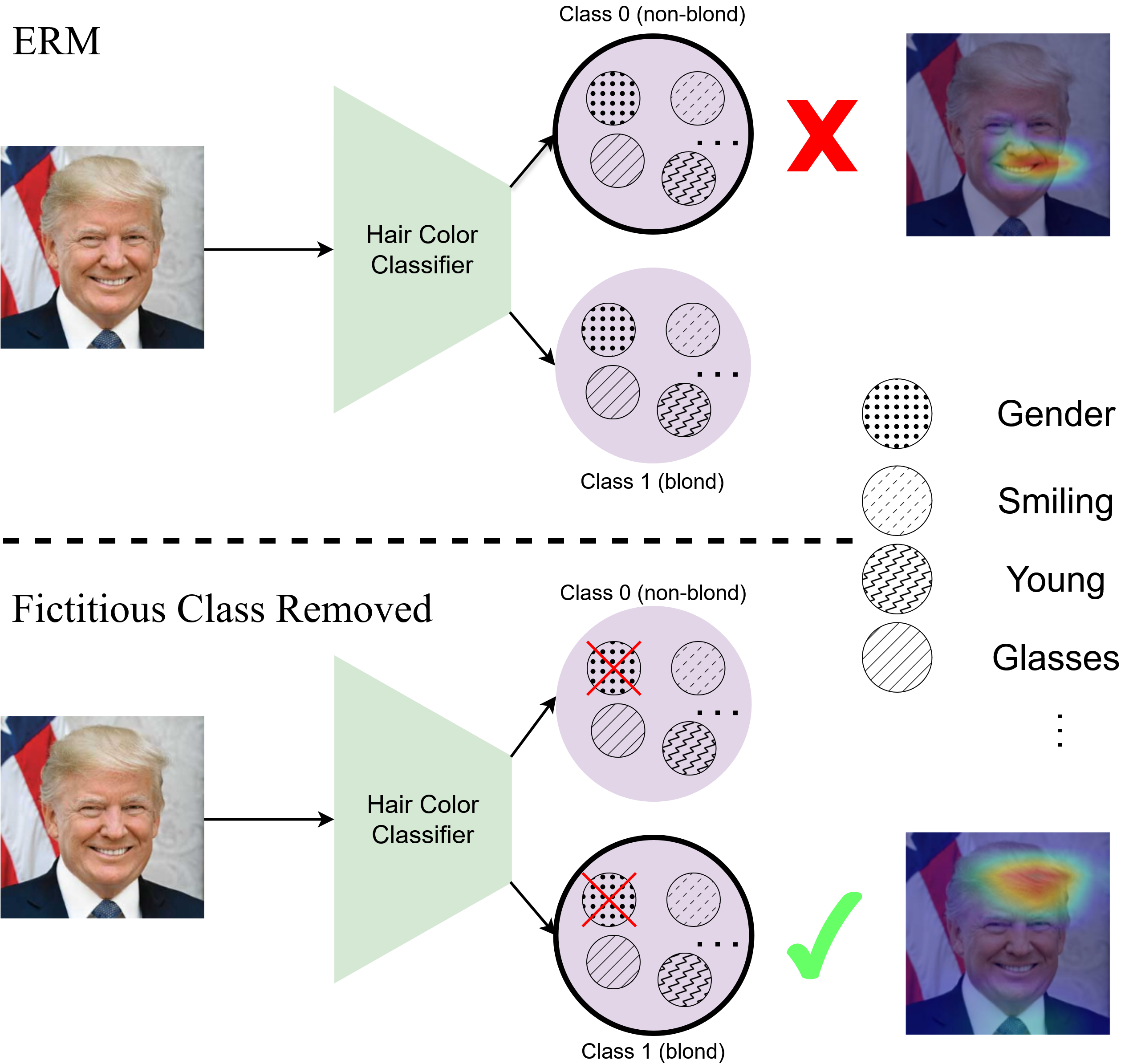}    
    \caption{Illustration of the adopted fictitious class perspective. Top: A non-robust Empirical Risk Minimization (ERM) classifier may rely on incorrect set of high-level features - fictitious sub-classes within a class - to mis-associate male gender to non-blond hair. Bottom: Removing the undesired fictitious (sub-)class from the set enables robust classification.
    }
    \label{fig:fig1} 
    \vspace{-3mm}
\end{figure}
\begin{figure*}[t]
    \centering
    \includegraphics[width=1\textwidth]{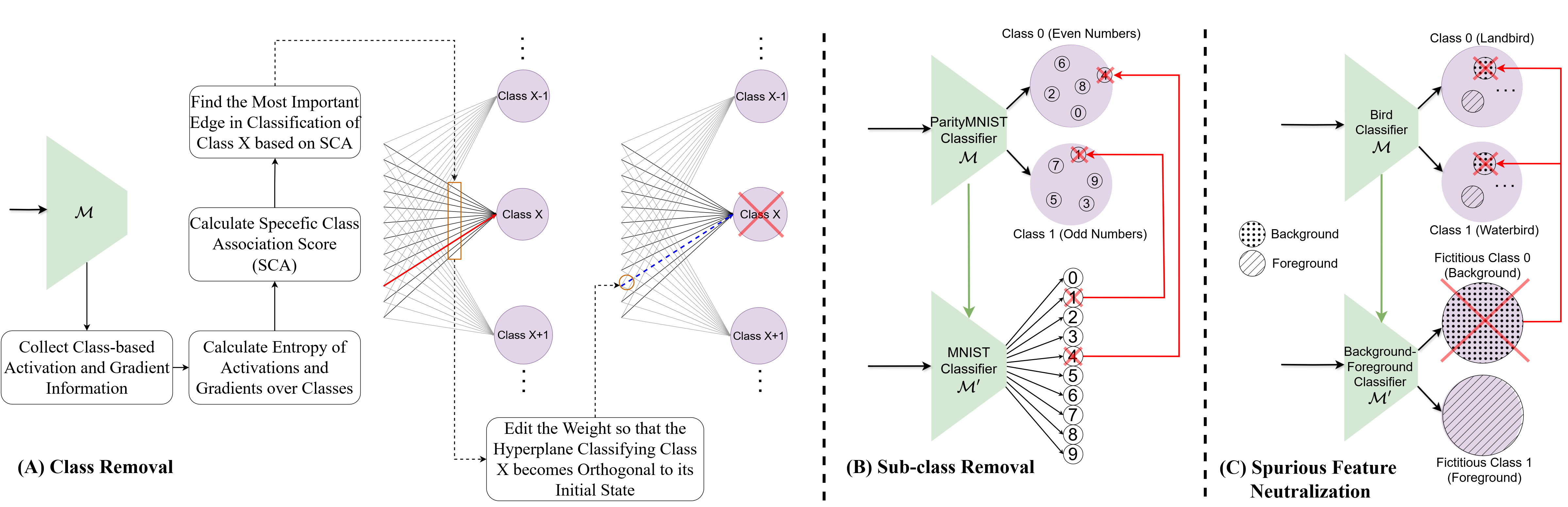} 
    \caption{(A) To remove a specific class, we first compute Specific Class Associate (SCA) score - a metric defined in this work jointly over the activations, gradients and entropy of the neural connections. The SCA score identifies the most important edge (red) for the class under consideration. We replace the weight of that edge with an analytically computed value such that the resulting decision hyperplane becomes orthogonal to its original state around the axis defined by the weight. (B) Any arbitrary sub-class for a model $\mathcal{M}$ can be removed by making a copy $\mathcal{M'}$ of the model and let it classify that sub-class as one of its main classes while all its parameters, except the last layer weights, are frozen. The sub-class can be removed from $\mathcal{M}$ with the same editing as that required for removing it from $\mathcal{M'}$. (C) Spurious features can be treated as fictitious sub-classes and removed using the same process as described in (B).}
    \label{fig:fig2}
\end{figure*}
In any case, existing techniques deal with  spurious correlation suppression in an ante-hoc manner. Leaving alone the viability and computational overhead of model retraining at the user's end;  where the model misbehavior due to spurious correlation is often first manifested, ante-hoc approaches may inadvertently compromise the overall model performance. 
Hence, they have limited practical value. 
Kirichenko et~al.~\cite{kirichenko2022last} showed that non-robust models also learn core/robust data features, albeit they lack strong reliance on them in decision making. This observation inspires us to retain the original learning of the model, thereby focusing on the possibility of post-hoc spurious correlation suppression by distracting a pre-trained model from paying too much attention to the spurious features.  

Another intriguing insight in the literature is presented by Eastwood et~al. \cite{eastwood2024spuriosity}, who argue that so-called spurious features are not entirely harmful. Aligned with \cite{eastwood2024spuriosity}, we posit that some apparent spurious correlations may even help model generalization, provided the right  underlying data distribution. For instance,  (spuriously) correlating a seagull with the background of sea might help correctly recognizing a bird over a sea as a seagull in a typical natural image setting. 
Clearly, \emph{spuriousness}  is a subjective notion, and techniques addressing spurious correlation  should provide the ability to  control the extent to which it can be neutralized for optimal  model performance. Unfortunately, existing  methods generally fail to explicitly account for such a control.   

In this work, we propose a post-hoc technique  that enables neutralizing the contribution of a high-level feature to model prediction by an arbitrary degree. We conceptualize high-level features as fictitious sub-classes within the original class. Our method withdraws model attention from a selected fictitious sub-class to control the model behavior (see Fig. \ref{fig:fig1}). 
Building on a strong theoretical foundation, we introduce a single-weight editing method to unlearn a fictitious class such that our post-hoc editing is applied to the network connection that contributes the most  to the original model behavior for that class. 
To find that connection, we analyze class-specific activations and gradients of the model for the given class. In addition to being a unique post-hoc method to address spurious correlation, our technique also does not require group-annotated samples for editing, as often required by some existing methods~\cite{sagawa2019distributionally, deng2024robust, kirichenko2022last}.

Our main contributions are summarized below.
\begin{enumerate}
\item We propose a first-of-its-kind post-hoc model unlearning technique to address spurious correlations. 
Our method edits only a single model weight to break off prediction reliance  on irrelevant high-level features in the input.
\item We provide theoretical foundations leveraging class activations and model gradients to single out the most significant model weight contributing to its behavior for a given (sub-)class. Editing this weight unlearns the target class with none-to-negligible negative impact on  the model performance for the remaining classes. 
\item Our method operates without group label annotations and, through extensive experiments, demonstrates state-of-the-art or comparable performance in mitigating spurious correlations while being the only post-hoc method.  
\end{enumerate}

\section{Related Work}
We discuss the key existing works by organizing them according to the aspects in which they relate to our approach.

\noindent\textbf{Spurious Correlation Mitigation:} To mitigate spurious correlation, early methods employed distributionally robust optimization (DRO),  which utilizes group annotations to up-weight the worst-group loss during optimization \cite{sagawa2019distributionally, hu2018does, oren2019distributionally, zhang2020coping, lu2025mitigating}. More recently, along similar lines, Deng et~al.~ \cite{deng2024robust} proposed robust model learning with progressive data expansion. Though effective in terms of achieved worst-group performance, its strong dependence on group annotations limits its practicality. To mitigate the issue,  other works proposed using only a limited amount of group annotated data \cite{sohoni2021barack, nam2022spread}. An  extreme scenario is where there is no group annotation available at all. Methods aiming at such a scenario \cite{creager2021environment, zhang2022contrastive, yang2024identifying} try to predict this information during the training process. In \cite{chakraborty2024exmap} utilized explainability heatmaps for clustering the groups. This solves the group label annotations requirement problem, however; similar to all  the methods mentioned above, \cite{chakraborty2024exmap} remains an ante-hoc technique. Recently, \cite{sun2025reverse} aimed mitigation of spurious correlations for unbiased scene graph generation, and \cite{wang2025backmix} explored the negative effect of foreground-background priors in open set recognition.

\noindent\textbf{Machine Unlearning:}  Mainly in response to data privacy protection regulations, Machine Unlearning \cite{bourtoule2021machine} has emerged as a field of study to address unlawful use of data in machine learning models. Although the task is trivial in some machine learning approaches like k-NN, merely requiring data deletion, it is seen as a major challenge in ANNs \cite{foster2024fast, chundawat2023can}. There are two broad approaches to machine unlearning. The first is \textit{Exact Unlearning}, which seeks efficient methods to retrain the model on responsible data to unlearn undesired concepts \cite{cao2015towards}.  \textit{Approximate Unlearning} \cite{NEURIPS2020_b8a65506, Graves_Nagisetty_Ganesh_2021, foster2024fast} aims at making the model as indistinguishable as possible to its counterpart that is trained without the undesired data \cite{kurmanji2024towards}. Although machine unlearning usually aims to undo the effects of some specific data on the model for privacy compliance, there are other recent approaches that also leverage this paradigm for bias mitigation \cite{chen2024fast} and eliminating the effect of corrupted data \cite{goel2024corrective}.

\noindent \textbf{Finding Significant Connections:}  The seminal work of Optimal Brain Damage \cite{lecun1989optimal} motivated the exploration to rank the neural network connections based on their importance in the classification task. This helps in reducing the memory footprint of the model by pruning the unimportant connections, which also leads to  better generalization and faster inference. This research direction is still active, pursuing model efficiency and performance gains by identifying the subsets of most important network connections  to retain  \cite{sunsimple, wu2024auto, khakineed, shi2024towards}. Our work is partially inspired by the counter-objective of seeking the most significant connection \textit{not to keep}, to enable our unique type of unlearning. 

\begin{figure}[t]
    \centering
    \includegraphics[width=0.35\textwidth]{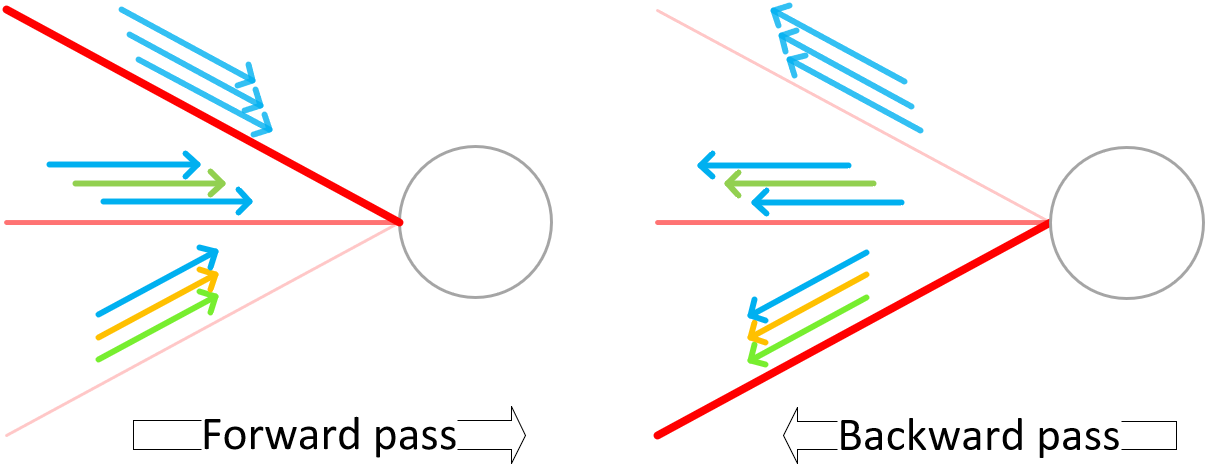} 
    \caption{Illustration of contrastive class association properties in forward and backward passes. Left: The edge receiving high activations for the samples of the same class (\textbf{blue} class) plays a more decisive role in discrimination - \textbf{bold red} edge - than the edges that get high activations for the samples of a higher number of classes (\textbf{green}, \textbf{blue}, and \textbf{orange} classes) - \textbf{pale red} edges. Right: The edge with high magnitude gradients from one class samples (\textbf{blue} class) is less crucial  for discrimination - \textbf{pale red} edges - than those which get higher magnitude gradients from the samples of several classes (\textbf{green}, \textbf{blue}, and \textbf{orange} classes), displayed as the \textbf{bold red} edge. Best viewed in digital form.}   
    \label{fig:forward_backward}
    \vspace{-3mm}
\end{figure}

\section{Problem Definition}
\label{sec:problem}
Consider a neural model representing a classification function 
\(f_{\theta}(.): \mathcal X \rightarrow \mathcal Y\) with parameters \(\theta \in \Theta\) to be trained on a  dataset \(\mathcal{D}_{tr} = \{(x_i, y_i)\}_{i=1}^n\), with training samples \(x_i \in \mathcal{X}\) and their corresponding class labels \(y_i \in \mathcal{Y}\). 
Let us denote a \textit{spurious feature} by \(\alpha \in A\), where \(A\) is the set of all presumed spurious features existing in \(\mathcal{D}_{tr}\). For our problem, a \textit{group} is defined using \(\alpha \in A\) and  \(y \in \mathcal{Y}\) as \(g := (\alpha, y) \in A \times \mathcal Y\) s.t.~\(g \in \mathcal{G}\), where \(\mathcal{G}\) is the set of all groups in \(\mathcal{D}_{tr}\). To suppress spurious correlation, the commonly sought objective \cite{nam2022spread} is to minimize 
\begin{equation}
    \label{eq:loss}
    \mathcal{L}_{worst\_group}(\theta) = \max_{_{g \in \mathcal{G}}} \mathbb{E}_{(x, y, \alpha) \sim P_g}[\ell(f_\theta(x), y)],
\end{equation}
where \(\mathcal{L}_{worst\_group}\) is the worst group loss, \(P_g\) is the group conditioned data distribution and \(\ell(.)\) is the model prediction loss.
Our specific objective is further defined as 
\begin{equation}
    \underset{\theta^* \in \Theta}{\arg\min}(\mathcal{L}_{worst\_group}(\theta))~\text{s.t.}~|| \theta^* - \theta||_0 \leq \delta.
    \label{eq:obj}
\end{equation}
In Eq.~(\ref{eq:obj}), $\theta^*$ constitutes the sought  vector of the model weights, $||.||_0$ denotes $\ell_0$-pseudo norm that counts the non-zero elements of the vector, and $\delta \in \mathbb Z^+$ is a pre-defined positive integer. In this work, we focus on $\delta = 1$ which enforces changing only a single model weight to suppress the spurious correlation impact on classification.

\section{Methodology}
\label{methodology}
\noindent{\bf Overview:} Figure~\ref{fig:fig2} shows an overview of our method which considers a high-level data feature as a  sub-class  within the actual class. Hypothetically, such features can  be classified by another classifier as its targets, hence their treatment as a class is well-justified. However, since we do not aim to actually classify them, we see them as  \textit{fictitious classes}.
This simple perspective allows us to treat spurious features as classes whose information can potentially be removed from the model by  machine unlearning, thereby enabling neutralization of the undesired spurious correlations previously learned by the model. 

For fictitious class removal, we propose a unique post-hoc technique that makes the hyperplane involved in classifying a fictitious class orthogonal to its original state. Considering our objective in Eq.~(\ref{eq:obj}), this transformation needs precision to ensure minimal changes to the original model. To that end, we restrict our class removal to modifying only a single weight, i.e., $\delta = 1$, that is associated with the most significant connection in the neural network for the fictitious class.  Such a connection should ideally be as exclusive as possible to the concerned class to minimize propagation of the editing effect to other classes. Hence, we also theoretically motivate and justify this exclusiveness for the connection identified in our approach.

\vspace{-2mm}
\subsection{Association of Neural Connections to Classes}
Here, we explore to the degree to which a  connection in the neural network contributes to the classification process both generally and class-specifically. The forward and backward passes are analyzed separately. We show that there are contrastive class associative properties in ANNs in the forward and backward passes. 

As formally posited in Theorem \ref{theorem:both} part~(a), in a forward pass, the edges that receive high activations from the samples of a small number of classes are more discriminative compared to those that get similar  activations from a larger number of classes. Based on  part (b) of Theorem \ref{theorem:both}, in backward pass, the edges with  associated high magnitude gradients for  samples from a fewer number of classes are less significant for classification as compared to those that have similar magnitude gradients from the samples of a larger number of classes. We  provide a simplified illustration of the phenomenon in Fig.~\ref{fig:forward_backward} where an extremal case for a three-class scenario is provided.

\begin{theorem}
\label{theorem:both}
Let \(e_{ji}^{(l)}\) and \(e_{qp}^{(l)}\) be two separate edges connecting nodes \(n_{i}^{(l)}\) and \(n_{p}^{(l)}\) in layer \(l\) to nodes \(n_{j}^{(l+1)}\) and \(n_{q}^{(l+1)}\) in layer \(l+1\) for a neural network trained on dataset \(\mathcal{D}_{tr}\), containing  samples \(\mathcal{S} = \{s_1, \dots, s_M\}\) to classify the set into classes \(\mathcal{C} = \{c_1, \dots, c_K\}\).

\noindent
\textbf{(a)} If in forward passes  \(e_{ji}^{(l)}\) receives high activations for the subset of samples \(\mathcal{S}_1 \subseteq \mathcal{S}\) from classes in \(\mathcal{C}_1 \subseteq \mathcal{C}\) and \(e_{qp}^{(l)}\) receives high activations for samples \(\mathcal{S}_2 \subseteq \mathcal{S}\) from classes in \(\mathcal{C}_2 \subseteq \mathcal{C}\), s.t.~\(|\mathcal{C}_1| > |\mathcal{C}_2|\) while \(|\mathcal{S}_1| = |\mathcal{S}_2|\), then the discriminative contribution of edge \(e_{ji}^{(l)}\) in the  induced model is less than that of the edge \(e_{qp}^{(l)}\).

\begin{figure}[t]
    \centering
    \includegraphics[width=0.45\textwidth]{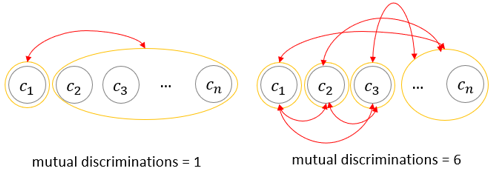} 
    \caption{An edge supports up to \(\binom{n+1}{2}\)  Mutual Discriminations (MDs) for \(n\) classes. Left: When gradient magnitude is large only for \(c_1\). Right: When gradient magnitudes for \(c_1\), \(c_2\), and \(c_3\) are large. We analyze MD to estimate edge contribution to  prediction.}
    \label{fig:mutual_discriminations}
\end{figure}

\noindent
\textbf{(b)} If in backward passes, \(e_{ji}^{(l)}\) receives high magnitude gradients for the set of samples \(\mathcal{S}_1 \subseteq \mathcal{S}\) from classes in \(\mathcal{C}_1 \subseteq \mathcal{C}\) and \(e_{qp}^{(l)}\) receives high magnitude gradients for samples \(\mathcal{S}_2 \subseteq \mathcal{S}\) from classes in \(\mathcal{C}_2 \subseteq \mathcal{C}\), and \(|\mathcal{C}_1| > |\mathcal{C}_2|\) while \(|\mathcal{S}_1| = |\mathcal{S}_2|\), then the contribution of edge \(e_{ji}^{(l)}\) is more than the contribution of edge \(e_{qp}^{(l)}\) in the classification task.
\end{theorem}

\begin{proof}
\noindent
    \textbf{(a)} Let the conditional entropy \(H(C|n)\) be the remaining uncertainty about the class given the activation of neuron \(n\),
\begin{equation}
    H(\mathcal{C}|n)=-\sum_{c\in \mathcal{C}}p(c|n)\log{p(c|n)}.
\end{equation}
So, for the neurons \(n_{i}^{(l)}\) and \(n_{p}^{(l)}\) we have
\begin{equation}
    \begin{split}
    H(\mathcal{C}|n_{i}^{(l)})=-\sum_{c\in \mathcal{C}_1}p(c|n_{i}^{(l)})\log{p(c|n_{i}^{(l)})} \\ = 
    -\sum_{|\mathcal{C}_1|}\frac{1}{|\mathcal{C}_1|}\log{\frac{1}{|\mathcal{C}_1|}}=\log|\mathcal{C}_1|
    \end{split}
\end{equation}
and
\begin{equation}
    \begin{split}
    H(\mathcal{C}|n_{p}^{(l)})=-\sum_{c\in \mathcal{C}_2}p(c|n_{p}^{(l)})\log{p(c|n_{p}^{(l)})}\\ = 
    -\sum_{|\mathcal{C}_2|}\frac{1}{|\mathcal{C}_2|}\log{\frac{1}{|\mathcal{C}_2|}}=\log|\mathcal{C}_2|.
    \end{split}
\end{equation}
Next, we calculate the amount of information gain \(\xi_{IG}\) after observing activation of the neurons \(n_{i}^{(l)}\) and \(n_{p}^{(l)}\) 
\begin{equation}
\label{eq:sup_12}
    \xi_{IG}(n_{i}^{(l)})=H(\mathcal{C})-H(\mathcal{C}|n_{i}^{(l)})=H(\mathcal{C})-\log|\mathcal{C}_1|,
\end{equation}
\begin{equation}
\label{eq:sup_13}
     \xi_{IG}(n_{p}^{(l)})=H(\mathcal{C})-H(\mathcal{C}|n_{p}^{(l)})=H(\mathcal{C})-\log|\mathcal{C}_2|.
\end{equation}
Since we have assumed in the theorem that \(|\mathcal{C}_1| > |\mathcal{C}_2|\),  from Eq.~(\ref{eq:sup_12}) and Eq.~(\ref{eq:sup_13}) we have
\begin{equation}
     \xi_{IG}(n_{i}^{(l)}) <  \xi_{IG}(n_{p}^{(l)}).
\end{equation}
So, receiving activations on \(e_{ji}^{(l)}\) reveals less information for the model for discriminating between the classes. 

\noindent
\textbf{(b)}
By definition, for a specific model parameter \(e\), if the gradient of a class specific loss w.r.t\text{.} to that parameter \(\frac{\partial{\mathcal{L}^{(c)}}}{\partial{e}}\) is high in magnitude, it shows that small changes in \(e\) have significant impact on distinguishing between the class \(c\) and other classes.
If the parameter exhibits high gradients for \(|\mathcal C_k|\) classes where \(\mathcal C_k \subseteq \mathcal C\), we can calculate the \textit{mutual discrimination} (MD) frequency (\(\nu_{MD}\)) based on the combination rule as \(\binom{|\mathcal C_k|}{2}\). Computing this for edges \(e_{ji}^{(l)}\) and \(e_{qp}^{(l)}\), and simplifying, we get  
\begin{equation}
\label{eq:sup_15}
   \nu_{MD}(e_{ji}^{(l)})=\frac{|\mathcal{C}_1|(|\mathcal{C}_1|-1)}{2},
\end{equation}
and 
\begin{equation}
\label{eq:sup_16}
    \nu_{MD}(e_{qp}^{(l)})=\frac{|\mathcal{C}_2|(|\mathcal{C}_2|-1)}{2}.
\end{equation}
Based on the theorem assumption \(|\mathcal{C}_1| > |\mathcal{C}_2|\), Eq.~(\ref{eq:sup_15}) and Eq.~(\ref{eq:sup_16})  give us 
\begin{equation}
    \nu_{MD}(e_{ji}^{(l)}) > \nu_{MD}(e_{qp}^{(l)}).
\end{equation}
So, as \(e_{ji}^{(l)}\) takes role in more mutual discriminations, it contributes more to the classification task compared to \(e_{qp}^{(l)}\). 
\end{proof} 

The apparently contrastive association of neural connections to classes in forward and backward passes, as identified above, can be explained intuitively as follows. In the forward pass, high activations for a large number of classes mean that the neuron is less helpful in discriminating between those classes. Hence, the associated connections do not contribute much to the classification. Conversely,  high magnitude gradients in the backward pass emulate high sensitivity of the edges to the associated classes. It helps more in narrowing down the classification decisions when the sensitivity remains high for more classes (see Fig. \ref{fig:mutual_discriminations}). 
\begin{SCfigure}[1][t]
    \centering
    \includegraphics[width=0.28\textwidth]{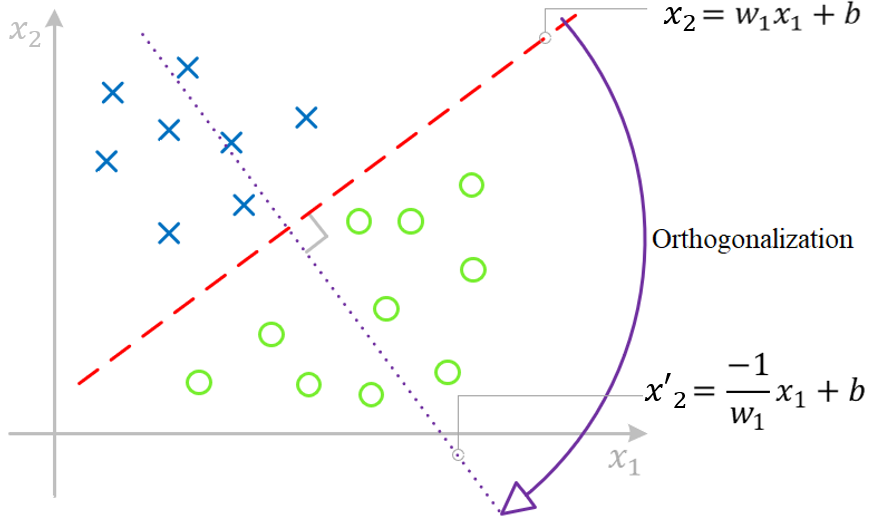} 
    \caption{\textbf{Orthogonalization} of the hyperplane classifying a class w.r.t\text{.} the most significant connection's weight, i.e., \(w_1\) - the only weight. It results in canceling the well-trained model's discriminatory ability.}
    \label{fig:orthogonalization}
\end{SCfigure}
\noindent{\bf Accumulative Class-wise Activations and Gradients:} 
In what follows, we rely on accumulative values of activation and gradient signals for the classes to develop our method. This enables us to sidestep any need to threshold the signals as high or low, which would be required if individual sample signals were considered. We compute class-wise accumulative activation \(\mathcal{A}_{ac.}^c\) for the edge \(e_{ji}^{(l)}\) as
\begin{equation}
\label{eq:acc_act}
    \mathcal{A}_{ac.}^c(e_{ji}^{(l)}) = \sum_{s \in S_c} a_j^{(l)}(s),
\end{equation}

\noindent where \(a_j^{(l)}(s)\) is the activation of \(n_{j}^{(l)}\) of the model receiving sample \(s\) of class \(c\) as the input. Similarly, we define the class-wise accumulative gradients \(\mathcal{G}_{ac.}^{c}\) for edge \(e_{ji} ^{(l)}\) as 
\begin{equation}
\label{eq:acc_grad}
    \mathcal{G}_{ac.}^c(e_{ji}^{(l)}) = \sum_{s \in S_c} \left\|\frac{\partial{\mathcal{L}^{(s)}}}{\partial{e_{ji}^{(l)}}}\right\|,
\end{equation}

\noindent where \(\frac{\partial{\mathcal{L}^{(s)}}}{\partial{e_{ji}^{(l)}}}\) is the gradient of loss function w.r.t\text{.} the edge \(e_{ji}^{(l)}\) of the model receiving sample \(s\) of class \(c\) as the input. More precisely, \(\mathcal{G}_{ac.}^{c}\) is a scalar value here, and we generalized the term \textit{gradient} for brevity. 

\noindent{\bf Class Association Score:} 
Considering Theorem \ref{theorem:both}, we first define a class association (CA) score. As per part (a) of the theorem, there is an inverse relationship between the entropy \(H\) of the accumulated activations of the classes and the contribution of the connection to classification.  Conversely, part (b) of Theorem \ref{theorem:both} suggests a direct relation between the entropy of accumulative gradient signals and their classification contribution. Hence, for a connection \(e_{ji}^{(l)}\), we define CA-score \(\Gamma_{CA}(e_{ji}^{(l)})\)  as follows

\begin{equation}
\label{eq:CA}
    \Gamma_{CA}(e_{ji}^{(l)}) = \frac{H(\bigcup_{c \in \mathcal{C}}\mathcal{G}_{ac.}^\mathcal{C}(e_{ji}^{(l)}))}{H(\bigcup_{c \in \mathcal{C}}\mathcal{A}_{ac.}^\mathcal{C}(e_{ji}^{(l)}))}.
\end{equation}

\noindent{\bf Specific Class Association Score:} 
We are also eventually interested in finding  which connections are more important in classifying certain classes. The CA-score in Eq.~(\ref{eq:CA}) provides a measure to associate a connection to the classes. To define a specific class association (SCA) score for an edge, we use the CA-score to scale the product of the accumulative class-wise activations and gradients with \(\Gamma_{CA}\).
\begin{equation}
\label{eq:SCA}
    \Gamma_{SCA}^c(e_{ji}^{(l)})=\Gamma_{CA}(e_{ji}^{(l)}) ~.~ \mathcal{G}_{ac.}^c(e_{ji}^{(l)}) ~.~ \mathcal{A}_{ac.}^c(e_{ji}^{(l)}),
\end{equation}
where $\Gamma_{SCA}^c(e_{ji}^{(l)})$ is the SCA-score for  $e_{ji}^{(l)}$ of class $c$.

\subsection{Neutralizing a Specific Class in Classifier}
\label{sec:NSCC}

Recall that spurious features are viewed as (sub-)classes in our method.  Hence, we are interested in precisely neutralizing specific classes in a classifier without destroying the classification hyperplanes for the others. To that end, we leverage the SCA-score to select the most important connection that contributes to classifying a specific class as exclusively as possible. We then make the (near) optimal classification hyperplane orthogonal to its initial state w.r.t\text{.} the axis corresponding to that connection. Fig.~\ref{fig:orthogonalization} illustrates the notion of  orthogonalization for the simplest case where there is only one connection, i.e.,~only one weight involved; and orthogonalizing the hyperplane w.r.t.~that  negates the  discriminative ability of the classifier. We leverage this concept in n-dimensional space in Theorem~\ref{theorem:orthogonalizer}. 
\begin{theorem}
\label{theorem:orthogonalizer}
    Let $\displaystyle y_i = \bm{w}_i^{\intercal} \displaystyle \bm{x} + \displaystyle b_i$ define the decision hyperplane for class $\displaystyle i$, and $\displaystyle w_{ji} \in \displaystyle \bm{w}_i$ be the connection weight with significant impact on the classification outcome. The discriminatory effect of  $\displaystyle w_{ji}$ can be nullified by applying the following transformation to it: 
    \begin{equation}
    \label{eq:orthogonolizer}
        f(w_{ji}) = -\frac{\Vert \bm{w}_i \Vert_2^2 - w_{ji}^2 + 1}{w_{ji}}.
    \end{equation}
\end{theorem}
\begin{proof}
    Let \(\bm{u} = [w_{1i}, \dots , w_{ni}, -1]\) be the normal vector of \(y_i\) and \(\bm{u}'\) be the normal vector of \(y'_i\) which is the resulting hyperplane by applying Eq.~(\ref{eq:orthogonolizer}) on $w_{ji}$. The dot product of \(\bm{u}\) and \(\bm{u}'\) is
    \begin{equation}
    \label{eq:17}
        \bm{u} \cdot \bm{u}' = \sum_{k=1}^{n} w_{ki}^2 - w_{ji}^2 + 1 + (-\frac{\Vert \bm{w}_i \Vert_2^2 - w_{ji}^2 + 1}{w_{ji}}.w_{ji}) = 0 \quad 
    \end{equation}
    Hence, the resulting hyperplane is orthogonal to its original state w.r.t.~$w_{ij}$. For a well-trained classifier, this implies that the classifier can no longer discriminate between the samples along the $w_{ij}$ dimension after the transformation. 
    Next, we calculate the difference caused by the transformation between the normal vectors of the two hyperplanes:
    \begin{equation}
    \label{eq:18}
        \bm{u} - \bm{u}' = \bm{v}_j(-\frac{\Vert \bm{w}_i \Vert_2^2 - w_{ji}^2 + 1}{w_{ji}} - w_{ji}),
    \end{equation}
    where \(\bm{v}_j\) is the one-hot vector associated with \(x_j\) axis. From Eq.~(\ref{eq:17}) and Eq.~(\ref{eq:18}) we conclude that the transformation in Eq.~(\ref{eq:orthogonolizer}) transforms the hyperplane to be orthogonal to its initial state merely with changes occurring along \(x_j\) axis.
\end{proof} 

\subsection{Removing Fictitious Classes}
\label{sub:method_remove_fictitious}
By definition, spurious features are (largely) unrelated to the causal features, which makes them a relatively high abstraction level  counterpart of interpretable causal features.   
We exploit this intrinsic high-level nature of spurious features to intuitively treat them as (sub-)classes within the original classes. 
In the form of Algorithm \ref{alg:1}, we have a tool to neutralize specific classes in a classifier. We adapt this tool further in Algorithm \ref{alg:2} to neutralize the impact of spurious correlations learned by a model by removing fictitious classes of spurious features. 

In Algorithm~\ref{alg:2}, to remove a fictitious class from the real classes of a model $\mathcal{M}$,  we create model $\mathcal{M'}$ that is a copy of $\mathcal{M}$. We alter the penultimate layer of the copy such that it has the spurious feature as one of its predicted classes. $\mathcal{M'}$ remains frozen, except for the weights of its penultimate layer. We fine-tune those weights   on the dataset that contains the spurious feature and its corresponding label. 
The purpose of this fine-tuning is to bottleneck  the spurious correlation learned by $\mathcal{M}$ to an identifiable connection in  $\mathcal{M'}$ - our experiments in Sec.~\ref{sec:experiments} show this can be achieved  efficiently.  Now that  $\mathcal{M'}$ recognizes the spurious feature as a class, we can remove it using the method in  Sec.~\ref{sec:NSCC}.

Since $\mathcal{M'}$ is the same as  $\mathcal{M}$ until the last layer, the high-level features extracted by both models are the same. We can apply the same class removal on $\mathcal{M}$ not to remove the complete class, but to remove the impact of the corresponding high-level features. This process can also be interpreted as removing a sub-class from  superclasses or removing a fictitious class from the real classes. Either way, it is notable that the method remains post-hoc because it does not require model retraining. The penultimate layer fine-tuning of $\mathcal{M'}$ is for weight identification purpose. The weight gets edited using Eq.~(\ref{eq:orthogonolizer}). It is emphasized that we intentionally present Algorithm~\ref{alg:2} such that  a copy of $\mathcal M$ gets created. This is to  clearly explain the underlying idea. Directly removing and replacing $\mathcal{M}$'s penultimate layer is a more memory-efficient alternative to implement the same concept.  

\label{sec:experiments}

\begin{algorithm}[t]
\caption{Class removal}
\label{alg:1}
\begin{algorithmic}[1]
    \REQUIRE Class index to remove (\(c_r\)), Model \(\mathcal{M}\)
    \FOR{edges \(e\) in the last layer \(L\) of \(\mathcal{M}\) } 
        \STATE Compute accumulative activations \(\mathcal{A}_{ac.}^c(e^{(L)})\)  using Eq.(\ref{eq:acc_act}) for \(\forall c \in \mathcal{C}\).
        \STATE Compute accumulative gradients magnitudes \(\mathcal{G}_{ac.}^c(e^{(L)})\) using Eq.~(\ref{eq:acc_grad}) for \(\forall c \in \mathcal{C}\).
        \STATE Compute CA score \(\Gamma_{CA}(e^{(L)})\) using Eq.~(\ref{eq:CA}). 
        \STATE Compute SCA score \(\Gamma_{SCA}^{c_r}(e^{(L)})\) using Eq.~(\ref{eq:SCA}).
    \ENDFOR

\STATE Select the most significant Connection \(e^*\) s.t. \\ \(e^* = \arg\max_{e} \Gamma_{SCA}^{c_r}(e^{(L)})\).  
\STATE Orthogonalize the hyperplane associated with \(c_r\) w.r.t\text{.} \(e^*\) following Eq.~(\ref{eq:orthogonolizer}).
\end{algorithmic}
\end{algorithm}
In general, removing a sub-class from the main classes is more challenging than directly removing a main class, as sub-class features might have much more inter-(sub-)class overlap. This is the reason that we continually sought a connection that is not only significant in classifying a certain fictitious class, but also does it as exclusively as possible. Nonetheless, even if we are able to find such a connection, it remains possible that the same connection also contributes to classifying other sub-classes  to some extent. 

To handle that, we define Partial Feature Neutralization (PFN). Conceptually, PFN  enables the hyperplane classifying a class to make a controllable, arbitrary tilt w.r.t.~its initial state. Eq.~(\ref{eq:partial}) shows the  weight update required to comply with PFN for feature neutralization to an arbitrary extent. 

\begin{equation}
    \label{eq:partial}
    f(w_{ji}) = r.(-\frac{\Vert \bm{w}_i \Vert_2^2 - w_{ji}^2 + 1}{w_{ji}}) + (1-r).w_{ji},
\end{equation}

where \(r \in [0, 1]\) is the neutralization rate. Using PFN,  we can controllably neutralize features to a level that eliminates model's over-reliance on them, without significantly altering  model's overall performance. This is established quantitatively in our experiments in \ref{sec:remove_spurious}. Optimal $r$ can be simply found using binary search, where we consider an $r$ as \textit{excessive} if it leads to a high reduction in average performance (compared to the improvement for the worst group) or if it changes the worst group, i.e.,~the accuracy of the initial worst group overtakes the average or another group's accuracy. 

\begin{algorithm}[t]
\caption{Fictitious class removal}
\label{alg:2}
\begin{algorithmic}[1]
    \REQUIRE Model $\mathcal{M}$,  fictitious class label $\Tilde{c}$
    \STATE Initialize $\mathcal{M}' \gets \text{Copy}(\mathcal{M})$ 
    \STATE Redefine the last FC layer of $\mathcal{M'}$ to have $\Tilde{c}$ as one of its classes
    \STATE Freeze all parameters of $\mathcal{M'}$ except the last FC layer weights
    \STATE Fine-tune $\mathcal{M'}$ to learn classifying $\Tilde{c}$
    \STATE Execute Steps 1-7 of Algorithm \ref{alg:1} on $\mathcal{M'}$
    \STATE Apply Step 8 of Algorithm \ref{alg:1} on $\mathcal{M}$
\end{algorithmic}
\end{algorithm}
\vspace{-2mm}
\section{Experiments}
\vspace{-1mm}
We first evaluate our unlearning method and verify its ability in effective and precise class removal. This is followed by its application in neutralizing spurious correlations. We start by testing the SCA score in Sec.~\ref{sec:class_removal}. In Sec.~\ref{sec:remove_fictitious}, we perform experiments on the more challenging problem of removing subclasses from classes. Finally, our proposed method, denoted as Fictitious Class Removal (FCR), is applied to neutralize spurious correlations in Sec.~\ref{sec:remove_spurious}. Analysis of CA score effectiveness and ablation study is provided in Sec.~\ref{subsec:ablation}.

For our experiments on (sub-)class removal, we define a simple two-layer convolutional model Conv-2Net, and also use ResNet-18 (\ref{sec:class_removal} and \ref{sec:remove_fictitious}). Aligned with the literature, ResNet-50 is used for extensive experiments on spurious correlation neutralization (\ref{sec:remove_spurious}). 

\vspace{-2mm}
\subsection{Class Removal}
\label{sec:class_removal}
We use the SCA-score in Eq.~(\ref{eq:SCA})  to rank weights according to their significance in class-specific classification and apply the transformation in Eq.~(\ref{eq:orthogonolizer}) to unlearn each class separately, without compromising the accuracy of the others. 

 Fig. \ref{fig:class_removal_LeNet} provides the accuracy curves for editing up to three weights from the model for removing CIFAR-10 classes individually.

We can see that our class removal method decreases the target class accuracy to almost zero with single weight editing.

Moreover, removing the class does not negatively impact the accuracies for other classes. For instance, after removing  class 0, the accuracy for classes 1, 3, 5, 6, and 7 remained intact, while there is a slight increase in the accuracy for classes 2, 4, 8, and 9. 

We also apply our class removal technique to remove multiple classes. Fig. \ref{fig:multiple_class_removal} summarizes the result for removing up to 8 random classes for CIFAR-10 and 98 classes for CIFAR-100 by editing just one weight per class. The average accuracy for the removed classes clearly drops to around random-guessing while there is no damage to the un-removed classes. These results clearly establish our approach as an effective class removal strategy. 

\vspace{-2mm}
\subsection{Fictitious Class Removal}
\label{sec:remove_fictitious}
\begin{figure*}[t]
    \centering
    \includegraphics[width=\textwidth] {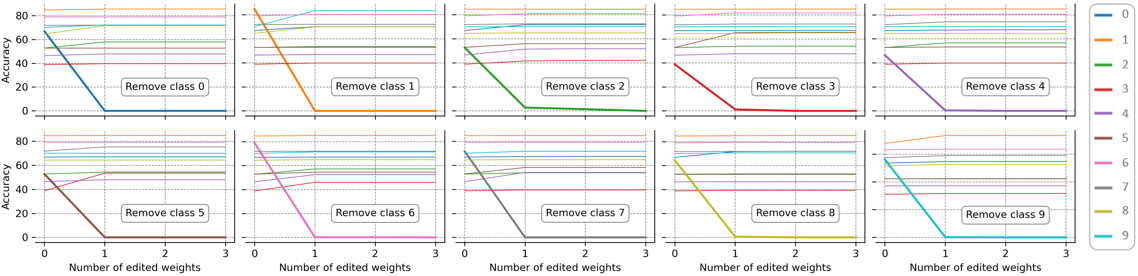}
    \caption{Removal of individual classes of CIFAR-10. In every case, the accuracy for the removed class decreases to almost zero by editing a single weight without negatively impacting accuracies for the other classes. Class labels are provided in the legend.}
    \label{fig:class_removal_LeNet}
\end{figure*}
\begin{figure}[h]
    \centering
    \includegraphics[width=0.47\textwidth]{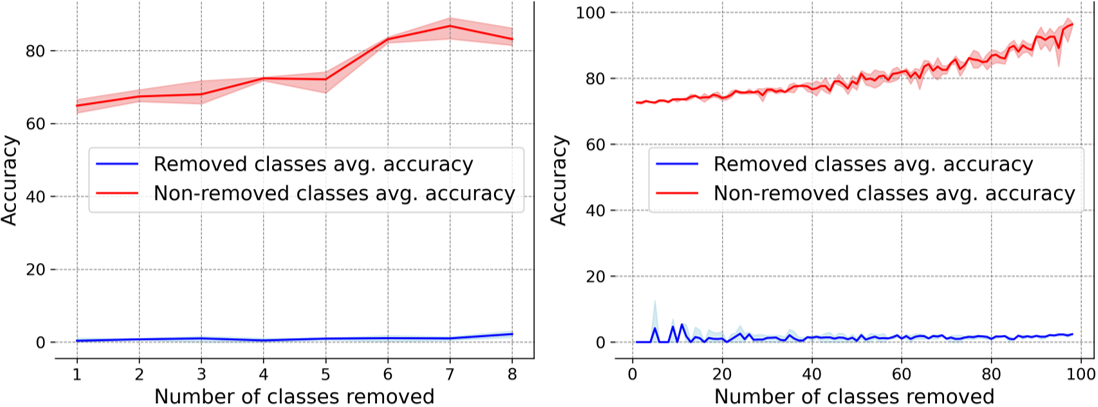} 
    \caption{Average accuracy change after removing random classes. Left: CIFAR-10 results using a well-trained Conv-2Net. 
    Right: CIFAR-100 results using ResNet-18. The accuracy for non-removed  classes improves as more classes are removed.}
    \label{fig:multiple_class_removal}
    \vspace{-2mm}
\end{figure}

Following the order of our discussion in Sec.~\ref{methodology}, here we examine the notion of removing fictitious classes from real classes. To this end,  we use ParityMNIST~\cite {mahinpei2021promises} dataset consisting of two classes; namely, `Class 0' and `Class 1' that respectively contain \textit{even} and \textit{odd} digits from 0 to 9. We assume that classes 0 and 1 are the true classes, i.e., superclasses,  while the digits within them are the fictitious classes, i.e., subclasses. Following our Algorithm \ref{alg:2}, a model $\mathcal{M}$ on ParityMNIST is trained. Then we tune a copy of $\mathcal{M}$ on MNIST. Since the digits are fictitious classes here, we can treat this model as $\mathcal{M'}$. We use the information of activations and gradients of the 10 digit classes to compute the SCA-Score for each class. To remove a fictitious class,  we perform the class removal as proposed in Sec.\ref{sec:class_removal}, i.e., using the SCA-score obtained from $\mathcal{M'}$ for the model $\mathcal{M}$ trained on  ParityMNIST.  

The collective results of removing all 10 fictitious classes are given in Fig. \ref{fig:remove_subclasses} (right) with different neutralization rates starting from 0 to 100 percent. It can be seen that almost all subclasses are removed without significant degradation in the accuracies for the other classes. To emphasize on further merits of our SCA metric, especially the components derived from Theorem \ref{theorem:both}, we also performed the same experiment directly using accumulative gradient magnitudes and accumulative activations instead of the SCA score. Results of that experiment are given in Fig. \ref{fig:remove_subclasses} (left). From the plots, it is clear that our proposed metric for selecting the most significant connections is more effective as it achieves lower average accuracy for the removed classes while maintaining slightly better average accuracy for the non-removed classes. Hence, our metric is utilized for the problem of spurious feature removal as well.  

\vspace{-2mm}
\subsection{Removing Spurious Features}
\label{sec:remove_spurious}

Finally, we provide results for removing spurious features. We treat this as removing fictitious classes from real classes. However, here we consider spurious high-level features  as the fictitious classes inside the true classes. We present results on the well-known datasets, commonly used in benchmarking spurious correlation mitigation methods, i.e., Waterbirds~\cite{sagawa2019distributionally}, CelebA~\cite{liu2015deep}, and Metashift~\cite{liang2022metashift}. 

\noindent{\bf Datasets:} A brief description of the dataset is given below.\\
\noindent\textit{Waterbirds} \cite{sagawa2019distributionally}: For this standard dataset, a  bird type classifier has to recognize waterbirds from landbirds. The dataset consists of four groups: waterbirds on water background, waterbirds on land background, landbirds on water background, and landbirds on land background. The majority groups are waterbirds on water background and landbirds on land background which comprise most of the dataset, while the other two groups where backgrounds and bird types do not match, are minority groups and consist a small proportion of the dataset. As a result of this imbalance, background is spuriously correlated to the bird types.

\noindent
\textit{CelebA} \cite{liu2015deep}: In CelebA, the gender feature has a spurious correlation to hair-color. The minority group in the dataset consists of samples of blond male individuals. So, the hair-color classifier tends to misclassify the hair-color of male and blond samples more than other groups.  

\noindent
\textit{MetaShift} \cite{liang2022metashift}: We employ  Cat vs.~Dog test from MetaShift dataset, where the Dog class is trained on bench and bike in the background, while the Cat class contains sofa and bed, and both classes are tested on samples having shelf in the background. In this standard setting, the challenge comes in the difference between  the test and  train data distributions.

\begin{figure}[t]
    \centering
    \includegraphics[width=0.47\textwidth]{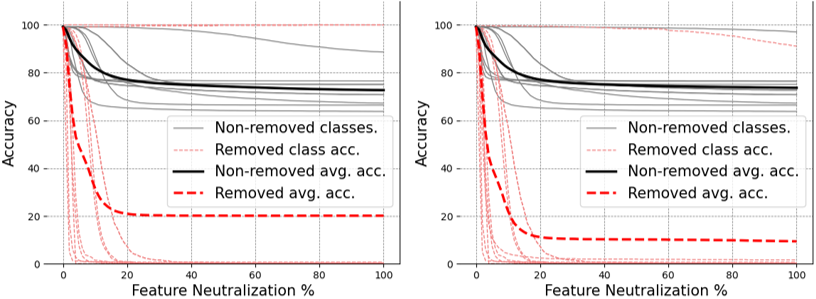} 
    \caption{Accuracy of removed and retained subclasses on ParityMNIST with different neutralization rate $r$. Left: Only accumulative gradients and activations are used for selecting the most significant connections. Right: SCA-score is used for selecting the connections.} 
    \label{fig:remove_subclasses}
    \vspace{-2mm}
\end{figure}

\begin{figure*}[t]
    \centering
    \includegraphics[width=0.85\textwidth]{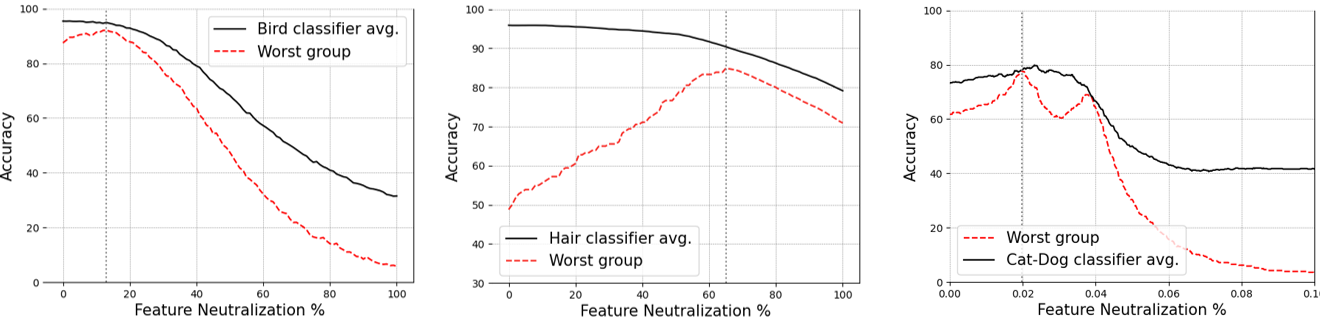} 
    \caption{Feature neutralization by different neutralization rate values ($r$). The best $r$ can
be found using binary search, where $r$ is considered $excessive$ if it reduces the average performance immensely or if it leads to change of the worst group, and is shown with a vertical dotted lines for Waterbirds (left), CelebA (middle), and MetaShift (right).}
    \label{fig:remove_spurious_curves}
\end{figure*}

\begin{figure}
    \centering
    \includegraphics[width=1\linewidth]{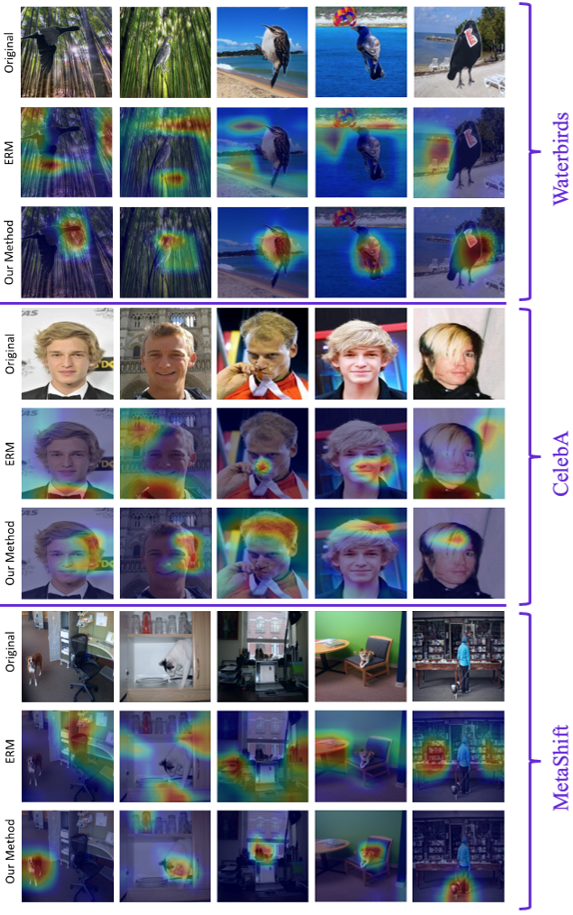}
    \caption{GradCAM saliency maps for representative  samples that are misclassified by ERM model and correctly classified by our method. The visualizations confirm that our method effectively remove the model's attention from non-robust features.}
    \label{fig:combined}
\end{figure}

\noindent{\bf Results:} We compare our method with the state-of-the-art methods of this direction, namely Group DRO \cite{sagawa2019distributionally}, PDE \cite{deng2024robust}, JTT \cite{liu2021just}, DFR \cite{kirichenko2022last}, LBC \cite{zheng2024learning}, DaC \cite{noohdani2024decompose}, MaskTune \cite{asgari2022masktune}, LC \cite{liu2022avoiding}, DISC \cite{wu2023discover}, and DFR+ExMAP \cite{chakraborty2024exmap}. The comparison results are given in Table~\ref{tab:SOTA}. Our method achieves the highest worst-group result for the Waterbirds and MetaShift datasets among all categories, 
whether or not they use Group Annotations, and on CelebA, it has the lowest gap between average and worst-group accuracies in the last category. As indicated, ours is the only method that can be applied post-hoc and it manipulates only one weight per class. Hence, it can be applied to already well-optimized models without requiring updates for excessive number of weights. 

Fig. \ref{fig:remove_spurious_curves} comprehensively presents the worst group and average accuracies using different neutralization rates $r$ for all the three datasets. We can choose optimal $r$ value using binary search as discussed in Sec. \ref{sub:method_remove_fictitious}. As observed in \ref{fig:remove_spurious_curves}, the optimal value for Waterbirds is $r=13$ (left) and for CelebA is $r=65$ (middle). The best $r$ for MetaShift is 0.0199 (right), on which the gap between the worst group and average accuracies is as low as $0.1\%$. The decline in worst group accuracy is observed after these points because further neutralization changes the worst group to another group since it still relies more on the removed feature.

To confirm that our technique is able to take away model's attention from spurious features and re-focus it on the causal features, we also observe the saliency maps of the models before and after our modifications.

In Fig.~\ref{fig:combined}, we  present representative qualitative results for Waterbirds, CelebA and MetaShift datasets. 
From the figure, it is clear that our method is able to fix the spurious correlation based associations in ERM based models very effectively. 
Since our method eventually disrupts the weight in the penultimate layer, it remains largely agnostic to backbone architectures. 
Our implementation will be made public after publication.

\begin{table*}[t]
\caption{Performance comparison of our FCR technique with the state-of-the-art methods, presenting worst-group (Worst) and gap (Gap)  between Worst and average accuracies on Waterbirds, CelebA and MetaShift datasets. The best results are shown in \textbf{bold}. Ours is a unique post-hoc method, which uses a single weight to remove spurious correlation for a given feature.} 

\label{tab:SOTA}
\centering
\begin{adjustbox}{valign=c, scale=0.85}
\begin{tblr}{
  column{3} = {c},
  column{4} = {c},
  column{5} = {c},
  column{6} = {c},
  column{7} = {c},
  column{8} = {c},
  column{9} = {c},
  column{10} = {c},
  cell{1}{1} = {r=2}{},
  cell{1}{2} = {r=2}{},
  cell{1}{3} = {c=2}{},
  cell{1}{5} = {r=2}{},
  cell{1}{6} = {c=2}{},
  cell{1}{9} = {c=2}{},
  cell{1}{12} = {c=2}{c},
  cell{3}{11} = {c},
  cell{3}{12} = {c},
  cell{3}{13} = {c},
  cell{4}{2} = {c},
  cell{4}{11} = {c},
  cell{4}{12} = {c},
  cell{4}{13} = {c},
  cell{5}{2} = {c},
  cell{5}{11} = {c},
  cell{5}{12} = {c},
  cell{5}{13} = {c},
  cell{6}{2} = {c},
  cell{6}{11} = {c},
  cell{6}{12} = {c},
  cell{6}{13} = {c},
  cell{7}{2} = {c},
  cell{7}{11} = {c},
  cell{7}{12} = {c},
  cell{7}{13} = {c},
  cell{8}{2} = {c},
  cell{8}{11} = {c},
  cell{8}{12} = {c},
  cell{8}{13} = {c},
  cell{9}{2} = {c},
  cell{9}{11} = {c},
  cell{9}{12} = {c},
  cell{9}{13} = {c},
  cell{10}{2} = {c},
  cell{10}{11} = {c},
  cell{10}{12} = {c},
  cell{10}{13} = {c},
  cell{11}{2} = {c},
  cell{11}{11} = {c},
  cell{11}{12} = {c},
  cell{11}{13} = {c},
  cell{12}{11} = {c},
  cell{12}{12} = {c},
  cell{12}{13} = {c},
  cell{13}{2} = {c},
  cell{13}{11} = {c},
  cell{13}{12} = {c},
  cell{13}{13} = {c},
  cell{14}{2} = {c},
  cell{14}{11} = {c},
  cell{14}{12} = {c},
  cell{14}{13} = {c},
  hline{1,15} = {-}{0.08em},
  hline{2} = {3-4,6-7,9-10,12-13}{0.03em},
  hline{3,5,9} = {-}{0.05em},
}
Method &  & Group Annotations\footnotemark[1] &  & Post-hoc & Waterbirds &  &  & CelebA &  &  & MetaShift (Cat vs. Dog) & \\
 &  & Train & Val. &  & Worst ($\%$) $\uparrow$ & Gap ($\%$) $\downarrow$ &  & Worst ($\%$) $\uparrow$ & Gap ($\%$) $\downarrow$ &  & Worst ($\%$) $\uparrow$ & Gap ($\%$) $\downarrow$\\
Group DRO (ICLR '20) \cite{sagawa2019distributionally} &  & Yes & Yes & \cross & 91.4 \scriptsize $\pm$1.1 & 2.1 &  & 88.9 \scriptsize $\pm$2.3 & 4.0 &  & 66.0 \scriptsize $\pm$3.8 & 7.6\\
PDE (ICML '23) \cite{deng2024robust} &  & Yes & Yes & \cross & ~90.3\scriptsize $\pm$0.3 & 2.1 &  & \textbf{~ 91.0\scriptsize $\pm$0.4} & 1.0 &  & - & -\\
JTT (ICML '21) \cite{liu2021just} &  & No & Yes & \cross & 86.7\scriptsize$\pm$ $N/A$ & 6.6 &  & 88.0\scriptsize $\pm$ $N/A$ & 6.9 &  & 64.6 \scriptsize $\pm$2.3 & 9.8\\
DFR (ICLR '23) \cite{kirichenko2022last} &  & No & Yes & \cross & 92.9\scriptsize $\pm$0.2 & 1.3 &  & 88.3\scriptsize $\pm$1.1 & 3.0 &  & 72.8 \scriptsize $\pm$3.8 & 4.7\\
LBC (IJCAI '24) \cite{zheng2024learning} &  & No & Yes & \cross & 88.1\scriptsize $\pm$1.4 & 6.0 &  & 87.4\scriptsize $\pm$1.8 & 5.0 &  & - & -\\
DaC (CVPR '24) \cite{noohdani2024decompose} &  & No & Yes & \cross & 92.3\scriptsize $\pm$ 0.4 & 3.0 &  & 81.9 \scriptsize $\pm$ 0.7 & 9.5 &  & \textbf{78.3 \scriptsize $\pm$ 1.6} & 1.0\\
Base (ERM) &  & No & No & \cross & 75.3\scriptsize $\pm$ 0.6 & 24.4 &  & 48.8 \scriptsize $\pm$ 1.1 & 47.1 &  & 62.1 \scriptsize $\pm$ 4.8 & 10.8\\
MaskTune (NeurIPS '22) \cite{asgari2022masktune} &  & No & No & \cross & 86.4\scriptsize $\pm$ 1.9 & 6.6 &  & 78.0\scriptsize $\pm$ 1.2 & 13.3 &  & 66.3 \scriptsize $\pm$ 6.3 & 6.8\\
LC (ICLR '23) \cite{liu2022avoiding}&  & No & No & \cross & 90.5\scriptsize $\pm$ 1.1 & $N/A$ &  & 88.1\scriptsize $\pm$ 0.8 & $N/A$ &  & - & -\\
DISC (ICML '23) \cite{wu2023discover}&  & No & No & \cross & 88.7\scriptsize $\pm$ 0.4 & 5.1 &  & - & - &  & 73.5\scriptsize $\pm$ 1.4 & 2.0\\
DFR+ExMap (CVPR '24) \cite{chakraborty2024exmap} &  & No & No & \cross & 92.5\scriptsize $\pm$ $N/A$ & 3.5 &  & 84.4\scriptsize $\pm$ $N/A$ & 7.4 &  & - & -\\
FCR - Ours &  & No & No & \tick & \textbf{93.2 \scriptsize $\pm$ 0.3} & 1.9 &  & 84.9\scriptsize $\pm$ 0.7 & 5.4 &  & \textbf{78.3 \scriptsize $\pm$ 0.4} & 0.1
\end{tblr}
\end{adjustbox}
\end{table*}

\begin{table}
\caption{Details of fine-tuning $\mathcal{M}'$ for different dataset.}
\begin{adjustbox}{valign=c, scale=0.8}
\centering
\begin{tblr}{
  column{even} = {c},
  column{3} = {c},
  hline{1,10} = {-}{0.08em},
  hline{2} = {-}{0.05em},
}
                                 & Waterbirds                                                                      & CelebA                                                                          & MetaShift                                                                       \\
\# of epochs                 & 5                                                                               & 5                                                                               & 100                                                                             \\
\# of params updated     & 4048                                                                            & 4048                                                                            & 4048                                                                            \\
Percentage of params updated & $ \sim 1.7 \times 10^{-2} \%$ & $\sim 1.7 \times 10^{-2} \%$ & $ \sim 1.7 \times 10^{-2}\%$ \\
Number of training samples       & 16,788                                                                          & 202,599                                                                         & 749                                                                             \\
Learning rate                    & $5 \times 10^{-6}$                                        & $10^{-3}$                                                                & $10^{-3}$                                                                \\
Optimizer                        & Adam                                                                            & Adam                                                                            & Adam                                                                            \\
Batch size                       & 32                                                                              & 32                                                                              & 16                                                                              \\
Accuracy on new task                        & $92.85\%$                                                    & $84.40\%$                                                    & $84.88\%$                                                    
\end{tblr}
\end{adjustbox}
\label{tab:finetune}
\end{table}
\vspace{-3mm}

\section{Further Discussions}
\label{sec:further}
\subsection{Experimental Setting}
In Sec.~\ref{sub:method_remove_fictitious}, we talked about the copy \(\mathcal M'\) of the original model \(\mathcal M\). In Algorithm~\ref{alg:2}, we also noted `fine-tuning' of \(\mathcal M'\). Here, we give further details about the fine tuning, as summarized in Table~\ref{tab:finetune}. The table summarizes the hyper-parameter settings used for fine-tuning model \(\mathcal{M}'\) for the three datasets used. The training data used for fine-tuning for CelebA is produced simply by using gender targets. For Waterbirds we used the `Water' and `Land' background images from Places365 \cite{zhou2017places} and used the foreground segmentation masks provided by \cite{sagawa2019distributionally}. For MetaShift, we utilized the segmentation technique proposed in \cite{chen2017rethinking} for generating foregrounds, and then filled the remaining image using \cite{telea2004image}. A few representative samples of the produced foreground-background dataset are shown in Fig.~\ref{fig:fore-back}.
\footnotetext[1]{By ``Group Annotations'', we mean the annotations of the original training data that include combinations of classes and features. Our method does not require these annotations and even the original training data at all. We only need to know what the spurious feature is, and we can use any data that can be used to fine-tune \(\mathcal{M}'\) to recognize the spurious feature as one of its targets.}

\begin{figure}[h]
    \centering
    \includegraphics[width=.4\textwidth]{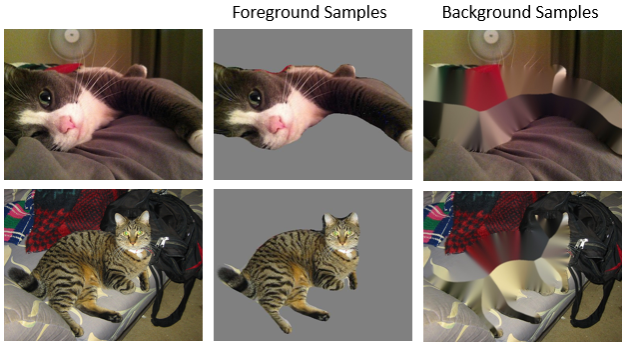} 
    \caption{Representative samples of the prepared foreground-background dataset to fine-tune $\mathcal{M}'$ for MetaShift.} 
    \label{fig:fore-back}
\end{figure}

\subsection{Ablation Study on CA-Score}
\label{subsec:ablation}
In Sec.~4.1, we introduced CA-score. This score is compositional in its nature, comprising components related to activations and gradients. 
We test the importance of the weights selected by the individual components of the CA-score and the eventual CA-score proposed in the paper. 

\begin{SCfigure}[1][h]
    \centering
    \includegraphics[width=0.33\textwidth]{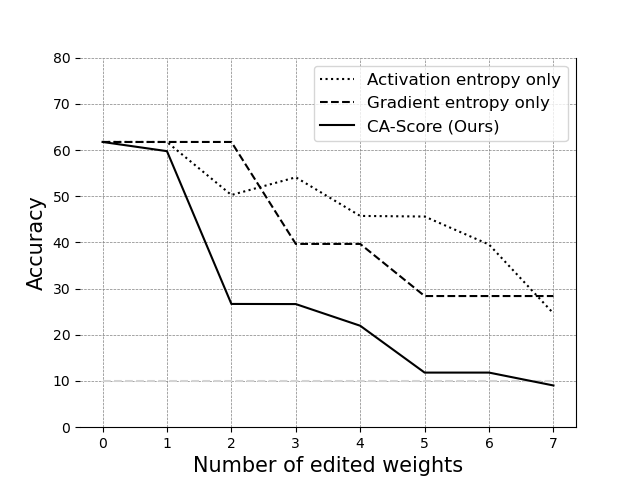} 
    \caption{Ablation for CA-Score. The significance of edges ranked by our composite CA-Score is considerably higher than to any of its constituent terms. The unlearning based on CA-Score is more effective.}
    \label{fig:sup_unlearning_1}
\end{SCfigure}

\begin{figure}[h]
    \centering
    \includegraphics[width=0.5\textwidth]{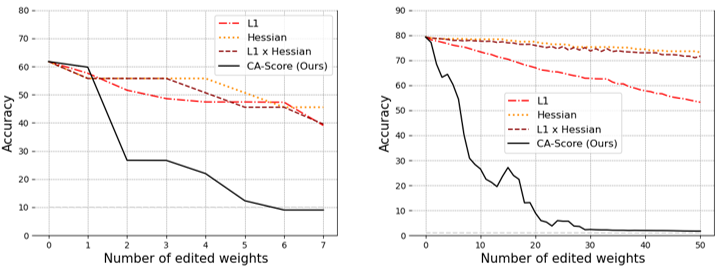} 
    \caption{Unlearning based on CA-Score compared to other common metrics of computing edge importance. (Left) Results on CIFAR-10. (Right) Results on CIFAR-100.}
    \label{fig:sup_unlearning_2}
\end{figure}

 Fig.~\ref{fig:sup_unlearning_1} gives the comparison results of the CA-Score and its constituent components, namely entropy of activations over classes (\(H(\bigcup_{c \in \mathcal{C}}\mathcal{G}_{acc.}^\mathcal{C}(e_{ji}^{(l)}))\)) and the inverse of entropy of gradients over classes (\(\frac{1}{H(\bigcup_{c \in \mathcal{C}}\mathcal{A}_{acc.}^\mathcal{C}(e_{ji}^{(l)}))}\)) on our Conv-2 model trained on CIFAR-10. It is evident that CA-score, as a composite metric, clearly achieves more favorable results in impacting the model performance through the weights specified by its eventual computation. 
  
 We also take a step further and compare the performance of removing weights using our CA-score with popular exiting techniques of using \(L_1\)-norm and Hessian to specify the most significant model weights.
 \begin{figure}[h!]
    \centering
    \includegraphics[width=.4\textwidth]{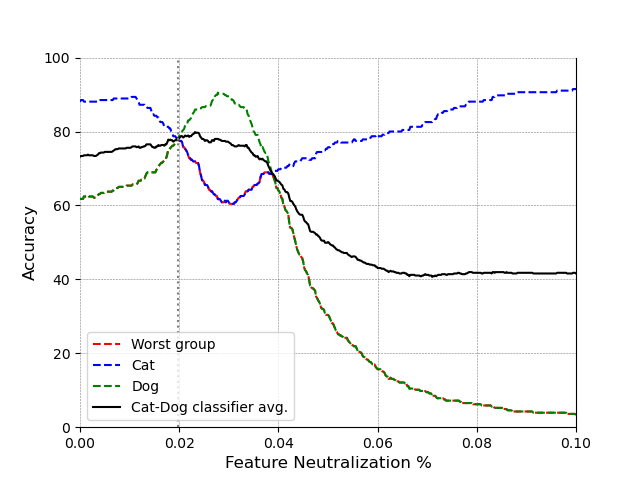} 
    \caption{Changes of the worst group on MetaShift. Excessive neutralization rate initially changes the worst group from \textit{Dog} to \textit{Cat}, and eventually larger values corrupt the model.}
    \label{fig:worst_group_change}
    \vspace{-3mm}
\end{figure}
 In Fig.~\ref{fig:sup_unlearning_2}, results are shown  for Conv-2Net on CIFAR-10 dataset and ResNet-18 on CIFAR-100 dataset. We gradually remove the most important weights suggested by our CA-score and the existing techniques, i.e., using \(L_1\)-norm, Hessian and \(L_1\)-norm times Hessian. It is observable that the CA-score generally perform much better than using conventional methods to specify the most important weights.

\subsection{Details on Subclass Removal on ParityMNIST}
Fig.~8 presented results on ParityMNIST. To further elaborate on these results, we expand on them in Fig.~\ref{fig:remove_from_parity_seperated}. It can be seen that with the exception of sub-class 3, all other subclasses are removed effectively with relatively low compromise in the accuracy for other subclasses.

We conjecture that the reason for low performance (i.e., higher accuracy despite removal) on subclass 3 stems in the fact that this class does not truly have exclusive high-level features as other digits. For example, subclass 3 shares its high-level features extensively with sub-class 8. Hence, the objective of  neutralizing 3 exclusively remains ill-formed. This also highlights an important fact that our general approach of editing the model with limited (or just one) weight may face restrictions if the features of the underlying classes are not exclusive to other classes. 
\begin{figure*}[t]
    \centering
    \includegraphics[width=1\textwidth]{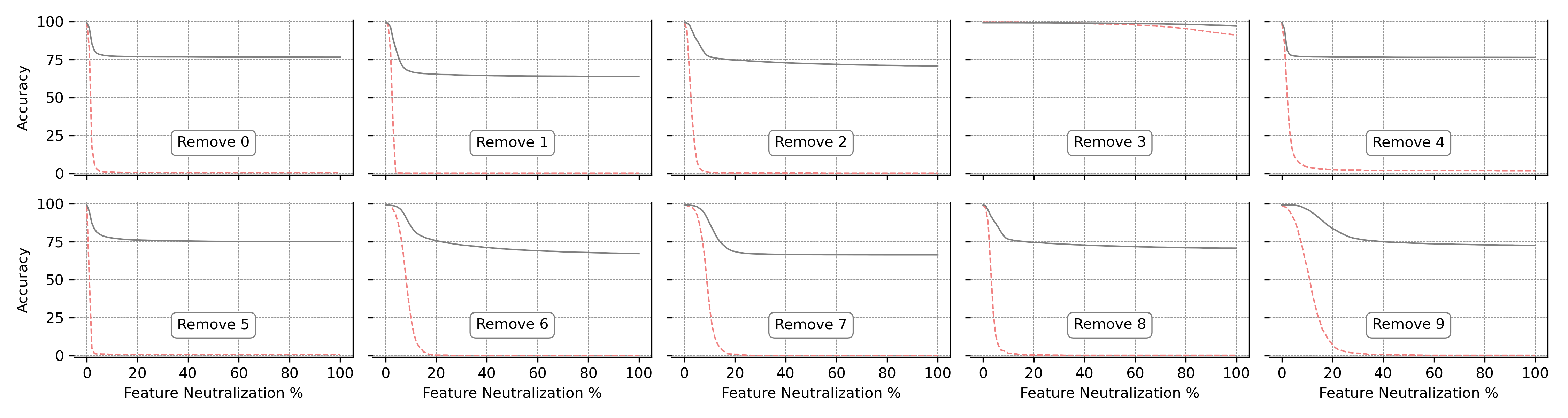} 
    \caption{Performance of subclass removal on parityMNIST. The dashed red and the solid gray curves represent the removed subclass accuracy and the non-removed subclasses accuracy average, respectively. The subclasses are removed effectively without compromising other subclasses,  except for digit `3', which could not be removed completely. We conjecture that this is due to the fact that `3' lacks exclusive high-level features. The set of high-level features of `3' is a subset of the set of the high-level features of `8'.}
    \label{fig:remove_from_parity_seperated}
\end{figure*}

\begin{figure*}[h]
    \centering
    \includegraphics[width=0.99\textwidth]{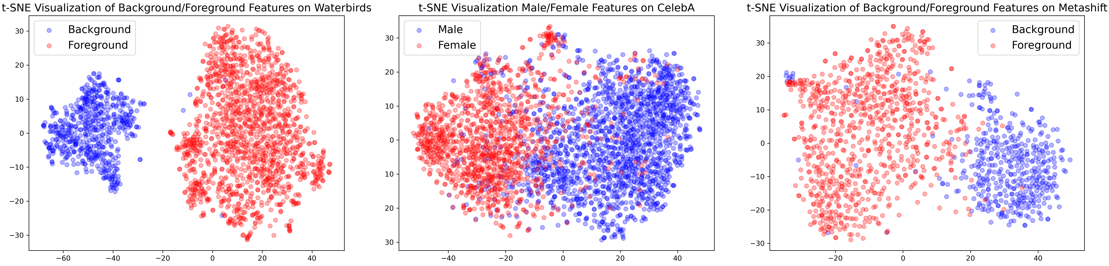} 
    \caption{t-SNE visualization of features for Waterbirds, CelebA, and MetaShift. Waterbirds and MetaShift show more disentagled embeddings compared to CelebA, which result in relatively lower optimal $r$ and higher spurious features neutralization performance on these datasets.}
    \label{fig:t-SNE}
\end{figure*}

\subsection{Changes of the Worst Group}
Extremely high neutralization rates lead to change in the worst group and eventually corruption of the model. Fig.~\ref{fig:worst_group_change} gives the group-wise accuracy for the case of Metashift with different values of \(r\). The average accuracy and the accuracies for each of the classes increase for \(r\) values up to approximately 0.01, which shows the proposed method is beyond merely establishing a tradeoff between the accuracies of different groups.

\subsection{t-SNE Visualizations}
For better understanding of why the optimal neutralization rate (\(r\)) differs among datasets, and why our method is more successful on some datasets compared to others, we utilized t-SNE visualizations (see Fig.~\ref{fig:t-SNE}). The visualizations show more separable spurious features for Waterbirds and MetaShift in comparison to CelebA in which features are more intertwined. This observation is well aligned with higher $r$ needed for our method for CelebA and better results for Waterbirds and MetaShift, as the spurious features can be neutralized more precisely. Additionally, it shows the potential of well-disentangled embeddings to help the proposed method to mitigate spurious features.

\subsection{Independence from Underlying Architectures}
\label{subsec:independence}
In this paper, we conducted experiments on ResNet architecture for compliance with the literature. Nevertheless, we claimed that the proposed spurious correlation is agnostic to the model architecture as it works on the last fully connected layer. Our method tries to neutralize the high-level spurious features that are extracted by the preceding layers. So, it is not affected by different procedures of inference in various architectures. To verify this claim experimentally, we adapted our FCR technique to ViT-B/16 models trained on Waterbirds and CelebA datasets taken from \cite{ghosal2024vision} (see Table~\ref{tab:vit}).

\begin{table}[]
\centering
\caption{Results on ViT-B/16. Worst group and average accuracies are reported for Waterbirds and CelebA. The improvement on worst groups while low degradation in averages verifies the effectiveness of our method on ViT architecture.}
\label{tab:vit}
\resizebox{0.45\textwidth}{!}{%
\begin{tabular}{lcccc} 
\toprule
     & \multicolumn{2}{c}{Waterbirds} & \multicolumn{2}{c}{CelebA} \\ \cmidrule{2-5} 
     & Worst (\%)          & Avg. (\%)           & Worst        & Avg. (\%)          \\ \midrule
ERM  & 89.30 \scriptsize$\pm$1.95        & 96.75 \scriptsize$\pm$0.05        & 94.10 \scriptsize$\pm$0.51        & 97.40 \scriptsize$\pm$0.42          \\
Ours & 93.55 \scriptsize$\pm$ 0.89         & 96.16 \scriptsize$\pm$0.16         & 95.42 \scriptsize$\pm$0.58       & 97.08 \scriptsize$\pm$ 0.79             \\ \bottomrule
\end{tabular}%
}
\end{table}

\section{Conclusion}
\label{sec:conclusion}
This work established an effective connection between the directions of post-hoc ``Machine Unlearning'' and ``Spurious Correlation Mitigation''. We proposed a framework that considers spurious features as fictitious classes inside real classes so that they can be mitigated using machine unlearning. Our unlearning technique modifies only a single weight of the original model for removing any subclass (fictitious class) from its superclass (real class). 
We also accounted for the level of spuriousness of features, and enabled   controllably neutralizing the impact of features that are likely spurious to a fair level. Our method figures out the most significant connections in classifying fictitious classes using a proposed metric relying on activations, gradients and entropy of the neural connections. Our theoretical insights are corroborated with empirical results, which also show  competitive performance for mitigating spurious correlation on three standard datasets\text{.}

\section*{Acknowledgments}
This research is supported by the Australian Government Research Training Scholarship. Dr. Naveed Akhtar is a recipient of the ARC Discovery Early Career Researcher Award (project \#DE230101058), funded by the Australian Government. Professor Ajmal Mian is the recipient of an ARC Future Fellowship Award (project \#FT210100268) funded by the Australian Government.

\bibliographystyle{IEEEtran}
\bibliography{main}

\newpage

\section{Biography Section}

\begin{IEEEbiography}
[{\includegraphics[width=1in,height=1.25in,clip,keepaspectratio]{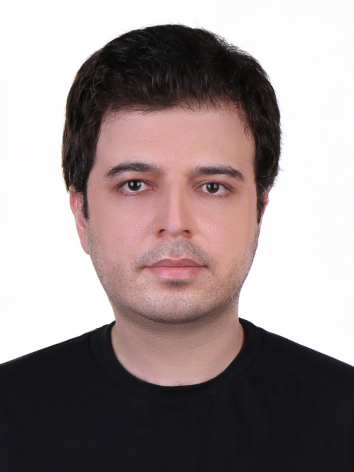}}]{Shahin Hakemi} is currently pursuing a Ph.D. in Computer Science in the Department of Computer Science and Software Engineering at the University of Western Australia. He was awarded the Australian Government Research Training Program (RTP) Scholarship in 2024. His research interests include explainable deep learning and computer vision.
\end{IEEEbiography}
\begin{IEEEbiography}[{\includegraphics[width=1in,height=1.25in,clip,keepaspectratio]{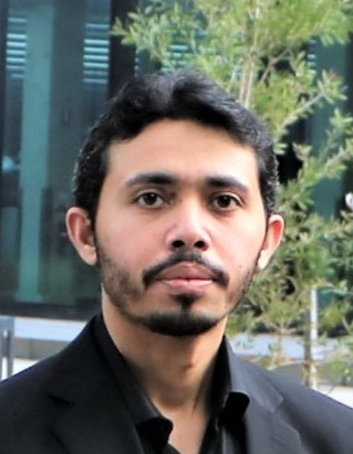}}]{Naveed Akhtar}
(Member, IEEE) received the master’s degree from Hochschule Bonn-Rhein-Sieg, North Rhine-Westphalia, Germany, and the Ph.D. degree in computer science from the University of Western Australia, Crawley, WA, Australia. He is currently a Senior Research Fellow with the University of Melbourne, Parkville, VIC, Australia. From 2021 to 2024, he was an ACM Distinguished Speaker. He was the recipient of the Discovery Early Career Researcher Award from the Australian Research Council, Universal Scientific Education and Research Network Laureate in Formal Sciences, and the Google Research Scholar Program Award in 2023. He was also a finalist for Western Australia's Early Career Scientist of the Year in 2021. He is an Associate Editor for IEEE Transactions on Neural Networks and Learning Systems, and has served as an Area Chair for prestigious conferences such as the IEEE Conference on Computer Vision and Pattern Recognition (CVPR) and the European Conference on Computer Vision (ECCV) on multiple occasions.
\end{IEEEbiography}
\begin{IEEEbiography}
[{\includegraphics[width=1in,height=1.25in,clip,keepaspectratio]{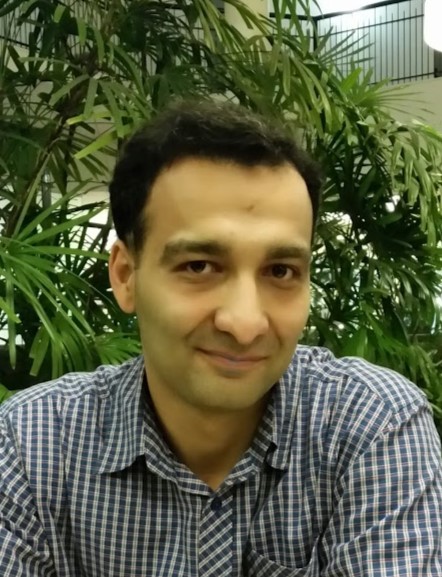}}]{Ghulam Mubashar Hassan} (Senior
Member, IEEE) received the B.S. degree from the University of Engineering and Technology, Peshawar, Pakistan, the M.S. degree from Oklahoma State University, USA, and the Ph.D. degree from The University of Western Australia (UWA). He is currently a faculty member in the Department of Computer Science and Software Engineering at UWA. His research interests include artificial intelligence, machine learning, and their applications in multidisciplinary problems. He is a recipient of multiple teaching excellence and research awards.
\end{IEEEbiography}
\begin{IEEEbiography}
[{\includegraphics[width=1in,height=1.25in,clip,keepaspectratio]{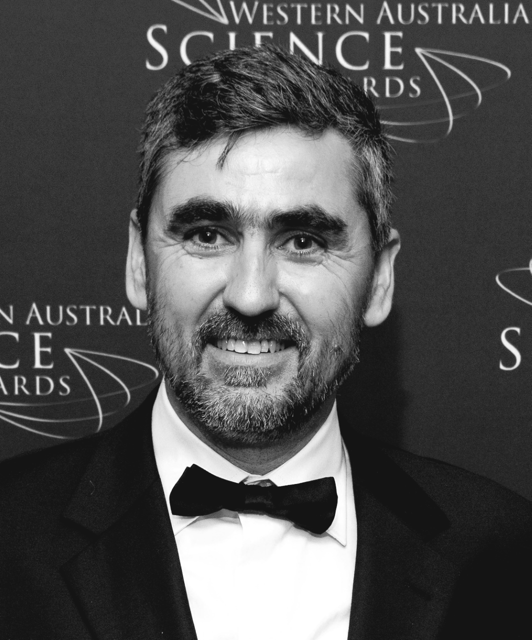}}]{Ajmal Mian}
(Senior Member, IEEE) is a professor of computer science
with the University of Western Australia. He has
received several awards including the West Aus-
tralian Early Career Scientist of the Year Award, the
Aspire Professional Development Award, the Vice-
chancellors Mid-career Research Award, the Out-
standing Young Investigator Award, IAPR Best Sci-
entific Paper Award, EH Thompson Award, and excel-
lence in research supervision award. He has received
several major research grants from the Australian Re-
search Council and the National Health and Medical
Research Council of Australia with a total funding of more than 13 Million.
He has served as an associate editor of IEEE Transactions on Neural Networks
and Learning Systems, IEEE Transactions on Image Processing and the Pattern
Recognition Journal. His research interests include computer vision, machine
learning, 3D shape analysis, human action recognition, video description and
hyperspectral image analysis.
\end{IEEEbiography}
\end{document}